\documentclass{article}

\usepackage[preprint]{neurips_2026}

\usepackage[utf8]{inputenc} 
\usepackage[T1]{fontenc}    
\usepackage{hyperref}       
\usepackage{url}            
\usepackage{booktabs}       
\usepackage{amsfonts}       
\usepackage{nicefrac}       
\usepackage{microtype}      
\usepackage{xcolor}         

\usepackage{amsthm, bm, amsmath}
\usepackage{tabularx}
\usepackage{adjustbox}
\usepackage{mathtools}
\mathtoolsset{showonlyrefs=true}
\newtheoremstyle{propstyle} 
    {2mm}                    
    {1mm}                    
    {\itshape}                   
    {}                           
    {\scshape}                   
    {.}                          
    {.5em}                       
    {}  
\theoremstyle{propstyle}
\newtheorem{theorem}{Theorem}
\theoremstyle{propstyle}

\theoremstyle{propstyle}
\newtheorem{proposition}{Proposition}
\theoremstyle{propstyle}

\theoremstyle{propstyle}

\theoremstyle{propstyle}

\theoremstyle{propstyle}

\newcommand{\bv}{\mathbf{v}}

\newcommand{\bx}{\mathbf{x}}
\newcommand{\by}{\mathbf{y}}

\newcommand{\bg}{\mathbf{g}}

\newcommand{\bz}{\mathbf{z}}

\newcommand{\bP}{\mathbf{P}}
\newcommand{\bA}{\mathbf{A}}

\newcommand{\bI}{\mathbf{I}}

\newcommand{\bV}{\mathbf{V}}

\newcommand{\bX}{\mathbf{X}}

\newcommand{\bfzero}{\mathbf{0}}

\newcommand{\bfmu}{\bm{\mu}}

\newcommand{\bfbeta}{\bm{\beta}}

\newcommand{\bfSigma}{\bm{\Sigma}}

\newcommand{\normal}{\mathcal{N}}
\newcommand{\kernel}{\mathcal{K}}
\newcommand{\domain}{\mathcal{D}}
\newcommand{\DL}{\mathcal{W}}
\newcommand{\defeq}{\overset{\text{def}}{=}}

\title{Permutation-preserving Functions and Neural Vecchia Covariance Kernels}

\author{%
  Jian Cao\\
  Department of Mathematics\\
  University of Houston\\
  Houston, TX 77047 \\
  \texttt{jcao21@central.uh.edu}\\
  \And
  Nian Liu\\
  Department of Mathematics\\
  University of Houston\\
  Houston, TX 77047 \\
  \texttt{nliu9@central.uh.edu}\\
  \AND
  Ying Lin\\
  Department of {Industrial and Systems Engineering}\\
  University of Houston\\
  Houston, TX 77047 \\
  \texttt{ylin53@central.uh.edu}\\
}

\begin{document}

\maketitle

\begin{abstract}
We introduce a novel framework for constructing scalable and flexible covariance kernels for Gaussian processes (GPs) by directly learning the covariance structure under a regression-type parameterization induced by Vecchia approximations, using deep neural architectures. Specifically, we model kriging coefficients and conditional standard deviations, deterministic quantities that uniquely characterize the covariance, providing stable and informative learning targets. Exploiting the permutation-equivariant structure of conditioning sets in the Vecchia factorization, we derive a universal representation for permutation-preserving functions and design neural architectures that respect this symmetry, leading to improved training stability and data efficiency. The proposed approach enables expressive, non-stationary kernel learning while maintaining computational scalability, thereby bridging classical GP methodology with modern deep learning.

\end{abstract}

\section{Introduction}

We study supervised learning problems where each response $y_{i}$ is associated with an input vector $\bx_{i} \in \mathbb{R}^d$. Gaussian processes (GPs) provide a principled and flexible framework for such settings by modeling functions through a covariance (kernel) function, which must be positive definite to define a valid joint distribution. A rich family of kernels has been developed, including classical stationary forms such as the squared exponential and Matérn kernels \citep{RasmussenWilliams}, as well as more expressive constructions like spectral mixture kernels \citep{WilsonAdams2013}.

Despite this progress, kernel selection remains a fundamental bottleneck. Real-world data often exhibit non-stationarity, heterogeneity, and high-dimensional interactions that are difficult to capture with standard parametric kernels. While non-stationary extensions, such as input warping \citep{snelson2003warped}, spatial deformations \citep{SampsonGuttorp1992}, and locally adaptive kernels \citep{PaciorekSchervish2006}, improve flexibility, it remains challenging to assess whether the assumed forms of non-stationarity adequately align with the underlying structure of the data.

Scalability poses a second, equally important challenge. Exact GP inference scales as $\mathcal{O}(N^3)$, where $N$ is the number of responses, making it impractical for large datasets. A large body of work has therefore focused on scalable approximations, including inducing point methods \citep{SnelsonGhahramani2006, Titsias2009}, kernel interpolation \citep{WilsonNickisch2015}, and sparse precision approaches \citep{Lindgren2011a}. Among these, the Vecchia approximation \citep{Vecchia1988} stands out for its accuracy and computation efficiency: it factorizes the $N$-dimensional joint density into a product of $N$ one-dimensional conditional distributions,
\[
f(\by) = \prod_{i=1}^N f(y_i \mid \by_{1:i-1}) \approx \prod_{i=1}^N f(y_i \mid \by_{c(i)}), \mbox{ where } c(i) \subset \{1, \ldots, i - 1\} \mbox{ and } |c(i)| \le m \ll N.
\]

This factorization suggests a natural route toward combining GPs with deep learning: one could model the conditional means and variances of $y_{i} | \by_{c(i)}$ using neural networks. However, this approach can be suboptimal for two reasons. First, given the covariance structure, the conditional means are still random since they depend on the random response vector $\by$, making them unstable learning targets. Second, conditional means alone do not adequately determine the covariance structure, preventing full recovery of the underlying dependency.

In this paper, we take a different route. Instead of modeling conditional means, we propose to learn the \textbf{kriging coefficients} together with the \textbf{conditional standard deviations}. These quantities are deterministic functions of the covariance structure, encode the same information as the covariance matrix, and provide a more stable and learnable parameterization. This shift turns the problem from learning random functionals to learning intrinsic feature of the covariance.

Our approach further exploits a key symmetry: the mapping from inputs to kriging coefficients is permutation-equivariant with respect to the conditioning set. Building on this observation, we derive a universal representation for \textbf{permutation-preserving functions}, analogous to Deep Sets \citep{zaheer2017deep}, and develop neural architectures that respect this structure. This leads to improved data efficiency and more stable training.

Our work connects to and differs from several lines of research on deep kernel learning. Prior approaches use neural networks to estimate kernel parameters \citep{sainsbury-dale_likelihood-free_2024, lenzi_neural_2023}, to transform inputs before applying parametric kernels \citep{Wilson2016DeepKernel, Wilson2016Stochastic, zhu2025scalable}, or to directly predict responses using augmented features \citep{Chen2021DeepKriging}. In contrast, we directly learn kriging coefficients as the fundamental objects governing the covariance. To our knowledge, this is the first work to bring kriging-coefficient-based kernel representations into deep learning for scalable GP modeling.

The remainder of the paper is organized as follows. Section~\ref{sec:cov_parameterization} introduces our covariance parameterization based on kriging coefficients and conditional standard deviations under the Vecchia approximation. Section~\ref{sec:neuvec} presents the proposed neural Vecchia architecture and develops the associated theoretical results for permutation-preserving functions. Section~\ref{sec:sim_study} evaluates the proposed method against three established covariance kernels using simulated datasets with varying covariance structures. Section~\ref{sec:application} applies the same set of methods to the Argo dataset, a large-scale and highly nonstationary oceanographic dataset. Finally, Section~\ref{sec:discuss} concludes with a discussion of the results and directions for future research.

\section{Covariance Reparameterization}
\label{sec:cov_parameterization}

Gaussian processes (GPs) are commonly used to model spatial or more general input-output relationships by specifying a mean function $\mu(\cdot)$ and a covariance kernel $\kernel(\cdot, \cdot)$ over a domain $\domain \subset \mathbb{R}^d$. Given a collection of input `locations' $\{\bx_i\}_{i=1}^{N} \subset \domain$, the associated responses $\by = (y_{1}, \ldots, y_{N})^\top$ follow a multivariate normal distribution, $\by \sim \normal(\bfmu, \bfSigma)$,
where $\bfmu \in \mathbb{R}^N$ with entries $\mu_i = \mu(\bx_i)$, and $\bfSigma \in \mathbb{R}^{N \times N}$ with entries $\bfSigma_{i,j} = \kernel(\bx_i, \bx_j)$. The covariance matrix $\bfSigma$ fully characterizes the dependence structure between entries of $\by$.

\subsection{Vecchia Approximation}

A major computational challenge in GP modeling is that evaluating the multivariate normal likelihood scales cubically in $N$. The Vecchia approximation \citep{Vecchia1988} addresses this issue by factorizing the joint density $f(\by)$ into 
$f(\by) \approx \tilde{f}(\by) = \prod_{i=1}^{N} f(y_i \mid \by_{c(i)})$,
where each conditioning set $c(i) \subset \{1, \ldots, i-1\}$ has cardinality at most $m$, a tuning parameter controlling the approximation accuracy. In practice, relatively small values of $m$ (e.g., $30 \le m \le 50$) are often sufficient for accurate approximation even when $N$ is large \citep{Datta2016, Katzfuss2017a, cao2025linear}.

The Vecchia approximation corresponds to another multivariate normal distribution $\normal(\bfmu, \tilde{\bfSigma})$, where the mean vector is unchanged and $\tilde{\bfSigma}$ is an approximation to the original covariance matrix $\bfSigma$. A key property of this approximation is that the inverse covariance admits a sparse Cholesky factorization. Specifically, there exists a sparse upper triangular matrix $\tilde{\bV}$ such that
$\tilde{\bfSigma} = (\tilde{\bV} \tilde{\bV}^\top)^{-1},$
where the nonzero entries in the $i$-th column of $\tilde{\bV}$ correspond to indices in $c(i) \cup \{i\}$ \citep{schafer2021sparse,Katzfuss2017a}. This sparsity reduces the computational complexity of likelihood evaluation and inference from $\mathcal{O}(N^3)$ to $\mathcal{O}(N m^3)$.

\subsection{Kriging-Coefficient Parameterization}

While the sparse inverse Cholesky factor $\tilde{\bV}$ provides a computationally efficient representation, it lacks direct interpretability. We therefore introduce an alternative covariance parameterization based on \textbf{kriging coefficients} and \textbf{conditional standard deviations}, which are more directly connected to likelihood computations.

Under the Vecchia approximation, each conditional distribution $f(y_i \mid \by_{c(i)})$ is univariate normal with mean and variance given by
\[
\mu_{i \mid c(i)} = \mu_i + \bfbeta_{i \mid c(i)}^\top (\by_{c(i)} - \bfmu_{c(i)})
\mbox{ and }
\sigma^2_{i \mid c(i)} \defeq \mbox{var}(y_i \mid \by_{c(i)}),
\]
where $\bfbeta_{i \mid c(i)} \in \mathbb{R}^{|c(i)|}$ is the kriging coefficients for predicting $y_i$ using its conditioning set. These quantities are directly related to the inverse Cholesky factor ( $\tilde{\bV}$ ) through
\begin{align}
\label{equ:new_vecc_parameterization_refined}
\bfbeta_{i | c(i)} = - \tilde{\bV}_{c(i), i} \tilde{\bV}_{i, i}^{-1}, \qquad 
\sigma_{i | c(i)} = \tilde{\bV}_{i, i}^{-1}.
\end{align}

\eqref{equ:new_vecc_parameterization_refined} establishes a one-to-one correspondence between the sparse inverse Cholesky factor and the collection of kriging coefficients and conditional standard deviations. In particular, the entire covariance structure (under the Vecchia approximation) encoded in $\tilde{\bfSigma}$ can be equivalently represented by the set $\{ \bfbeta_{i \mid c(i)}, \sigma_{i \mid c(i)} \}_{i=1}^N$. 
This representation offers two main advantages. First, it provides a more direct connection to the likelihood function: the kriging coefficients $\bfbeta_{i \mid c(i)}$ play a role analogous to regression coefficients in least-squares formulations, whereas working with the covariance matrix typically requires costly matrix factorizations to evaluate the log-likelihood. Second, the parameters are defined through low-dimensional conditional relationships, which naturally decompose across observations and make the model well suited for mini-batch optimization.

\section{Neural Vecchia approximations}
\label{sec:neuvec}

We propose a neural parameterization of Vecchia approximations by learning the kriging coefficients and conditional standard deviations using deep neural networks. These quantities uniquely determine the inverse Cholesky factor of the covariance matrix, and thus fully characterize the underlying covariance structure. Specifically, we introduce two neural architectures, denoted by $\DL^{\mu}$ and $\DL^{\sigma}$, to model the kriging coefficients $\hat{\bfbeta}_{i \mid c(i)}$ and the conditional standard deviation $\hat{\sigma}_{i \mid c(i)}$, respectively. Both networks take as input the spatial locations associated with the target index $i$ and its conditioning set $c(i)$. Formally, we define
\begin{align}
    &\DL^{\mu}: (\mathbb{R}^{m \times d}, \mathbb{R}^{1 \times d}) \rightarrow \mathbb{R}^{m} \mbox{ s.t. } \hat{\bfbeta}_{i|c(i)} = \DL^{\mu}(\bX_{c(i)}, \bX_{\{i\}}), \\ 
    &\DL^{\sigma}: (\mathbb{R}^{m \times d}, \mathbb{R}^{1 \times d}) \rightarrow \mathbb{R} \mbox{ s.t. } \hat{\sigma}_{i | c(i)} = \DL^{\sigma}(\bX_{c(i)}, \bX_{\{i\}}), 
\end{align}
where $\bX$ subscripted by an index set denotes the corresponding selection of rows from $\bX$. 

Importantly, the two mappings exhibit distinct symmetry structures with respect to their first arguments. The mapping $\DL^{\sigma}$ is permutation invariant, meaning that its output remains unchanged under any reordering of the conditioning locations. In contrast, $\DL^{\mu}$ is permutation preserving: permuting the conditioning locations induces the same permutation in the output kriging coefficients. Permutation invariance follows the standard definition used in \citet{zaheer2017deep}. The notion of permutation preservation is similarly natural and refers to equivariance between the ordering of the inputs and outputs; its formal definition is provided in Theorem~\ref{thm:perm_preserving}.

Meanwhile, the two-input formulation above can be rewritten as
\begin{equation}
\label{equ:DL_rewrite}
\DL\left([\bX_{\{j\}}, \; j \in c(i)], \bX_{\{i\}}\right) \rightarrow \DL\left(\big[[\bX_{\{j\}}, \bX_{\{i\}}], \; j \in c(i)\big]\right),
\end{equation}
where $\DL$ denotes either $\DL^{\sigma}$ or $\DL^{\mu}$, and $[\bX_{\{j\}}, \bX_{\{i}\}]$ represents the concatenation of the two vectors. Under this reformulation, the networks can be viewed as mappings
\begin{equation}
\label{equ:DL_rewrite_map}
\DL^{\mu}: (\mathbb{R}^{m \times 2d}) \rightarrow \mathbb{R}^{m} \qquad \DL^{\sigma}: (\mathbb{R}^{m \times 2d}) \rightarrow \mathbb{R},
\end{equation}
where $\DL^{\mu}$ and $\DL^{\sigma}$ are permutation-preserving and permutation-invariant, respectively, with respect to permutations of the rows of the input matrix. It is worth noting that the representation in \eqref{equ:DL_rewrite} is not unique. Any augmented representation containing equivalent information may be used. For example,
$[\bX_{\{j\}}, \bX_{\{i\}}] \Leftrightarrow [\mbox{dist}_{\{j\}, \{i\}}, \bv_{\{j\}, \{i\}}, \bX_{\{i\}}]$, where $\mbox{dist}_{\{j\}, \{i\}}$ and $\bv_{\{j\}, \{i\}}$ denote the Euclidean distance and the corresponding unit direction vector between $\bX_{\{j\}}$ and $\bX_{\{i\}}$, respectively. More generally, we denote the resulting augmented input domain by $\mathbb{R}^{m \times d_{A}}$, where $d_{A}$ is the dimension of the augmented feature representation. In the simplest case, $d_{A} = 2d$, although more structured feature representations may improve the training efficiency and expressive power of $\DL^{\mu}$ and $\DL^{\sigma}$.

To improve training efficiency and exploit the permutation-invariant structure of $\DL^{\sigma}$, we adopt the Deep Sets architecture of \citet{zaheer2017deep}. In particular, Theorem~9 of \citet{zaheer2017deep} shows that any continuous permutation-invariant function, including $\DL^{\sigma}$, can be approximated arbitrarily close by the form $\rho\big(\sum_{i = 1}^{m} \phi(\bX_{\{i\}})\big)$ for suitable transformations of $\phi$ and $\rho$. Moreover, when the input dimension is one-dimensional, this representation becomes exact. In this paper, we parameterize both $\phi$ and $\rho$ in $\DL^{\sigma}$ using fully connected neural networks. An illustration of the resulting architecture is provided in Figure~\ref{fig:W_sigma}.
\begin{figure}[htbp]
    \centering
    \includegraphics[width=0.9\textwidth]{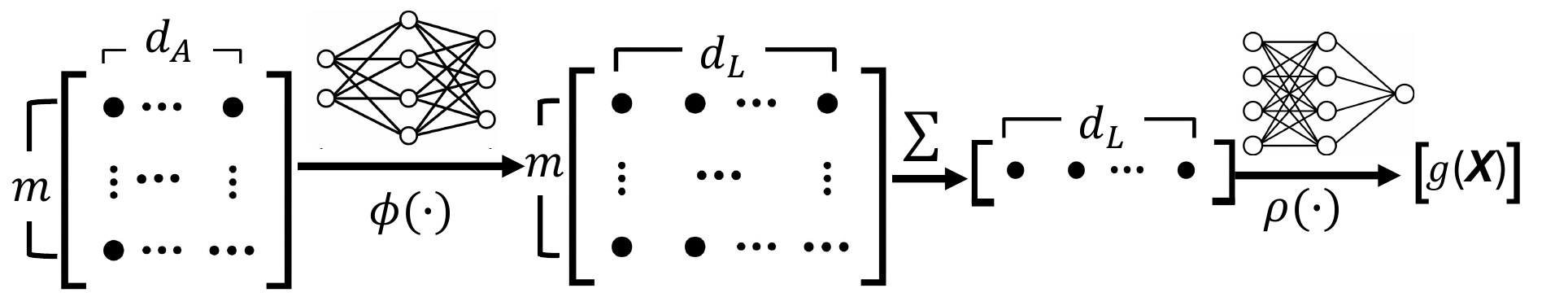}
    \caption{Illustration of the architecture of $\DL^{\sigma}$}
    \label{fig:W_sigma}
\end{figure}
Here, $d_{A}$ and $d_{L}$ denote the dimensions of the augmented input space and latent feature space, respectively, while $\sum$ represents summation over the rows of the input matrix (i.e., aggregation over the conditioning set). Throughout this paper, we focus on the basic augmentation strategy described in \eqref{equ:DL_rewrite}, under which $d_A = 2d$. The choice of latent dimension $d_L$ depends on the scale and complexity of the training data; concrete specifications are provided in Sections~\ref{sec:sim_study} and~\ref{sec:application}.

In this work, we extend Theorems 7 and 9 of \citet{zaheer2017deep} to vector-valued functions, which is essential for constructing a universal representation of permutation-preserving mappings in Theorem~\ref{thm:perm_preserving}. The resulting extension is summarized in Proposition~\ref{prp:perm_invariant}.
\begin{proposition}
    \label{prp:perm_invariant}
    $\bX \in \mathbb{R}^{m \times k_{1}}$ denotes a collection of $m$ input vectors from a compact subset of $\mathbb{R}^{k_{1}}$. Suppose $\bg: \mathbb{R}^{m \times k_{1}} \rightarrow \mathbb{R}^{k_{2}}$ is a continuous permutation-invariant function satisfying $\bg(\bP \bX) = \bg(\bX)$ for any permutation matrix $\bP$, it can be approximated arbitrarily close by $\rho\big(\sum_{i = 1}^{m} \phi(\bX_{\{i\}})\big)$, for suitable transformations $\phi$ and $\rho$. Specifically, when $k_{1} = 1$, $\bg$ exactly has the representation of $\rho\big(\sum_{i = 1}^{m} \phi(\bX_{\{i\}})\big)$. Here, $\rho$ is allowed to be vector-valued.
\end{proposition}
\noindent
The proof is provided in Appendix~\ref{app:proofs}. We now turn to $\DL^{\mu}$, which is permutation preserving rather than permutation invariant. Specifically, $\DL^{\mu}$ permutes its output according to the same permutation applied to its input. We show that this class of functions also admits a universal representation.
\begin{theorem}
    \label{thm:perm_preserving}
    $\bX \in \mathbb{R}^{m \times k_{1}}$ denotes a collection of $m$ input vectors from a compact subset of $\mathbb{R}^{k_{1}}$. Suppose $\bg: \mathbb{R}^{m \times k_{1}} \rightarrow \mathbb{R}^{k_{2}}$ is a continuous vector-valued permutation-preserving function satisfying $\bg(\bP \bX) = \bP \bg(\bX)$ for any permutation matrix $\bP \in \mathbb{R}^{m \times m}$, the $i$th component of $\bg$, denoted by $g_{i}$, can be approximated arbitrarily close by
    \begin{equation}
    \label{equ:perm_preserve}
    g_{i}(\bX) \approx \rho_{1}\left(\bX_{\{i\}}, \rho_{2}\left(\sum_{j \in \{1, \ldots, m\} \setminus \{i\}} \phi(\bX_{\{j\}})\right)\right),
    \end{equation}
    for suitable transformations $\phi$, $\rho_{1}$ and $\rho_{2}$. Moreover, when $k_{1} = 1$, the representation above is exact.
\end{theorem}
The proof of Theorem~\ref{thm:perm_preserving}, which builds upon Proposition~\ref{prp:perm_invariant}, is also provided in Appendix~\ref{app:proofs}. Intuitively, $\phi$ first maps each input into a latent representation, $\rho_{2}$ aggregates information from the remaining inputs, and $\rho_{1}$ models the interaction between the target input $\bX_{\{i\}}$ and the aggregated contextual information. Theorem~\ref{thm:perm_preserving} therefore provides a universal characterization of permutation-preserving functions through the three transformations $\phi$, $\rho_{1}$, and $\rho_{2}$, guiding the proposed neural architecture for learning kriging coefficients.

Finally, \citet{zaheer2017deep} conjectured that set functions defined on subsets of an uncountable universe with varying cardinalities can also be represented through the form $\rho\big(\sum_{\bx \in \bX} \phi(\bx)\big)$. Under this conjecture, both Proposition~\ref{prp:perm_invariant} and Theorem~\ref{thm:perm_preserving} naturally extend to varying conditioning sizes $m \in \mathbb{N}^{+}$. Consequently, the learned transformations $\phi$, $\rho_{1}$, and $\rho_{2}$ may generalize across different choices of $m$, potentially enabling covariance representations beyond the Vecchia framework.

Figure~\ref{fig:W_mu} illustrates the proposed universal architecture for permutation-preserving functions.
\begin{figure}[htbp]
    \centering
    \begin{adjustbox}{center}
    \includegraphics[width=1\textwidth]{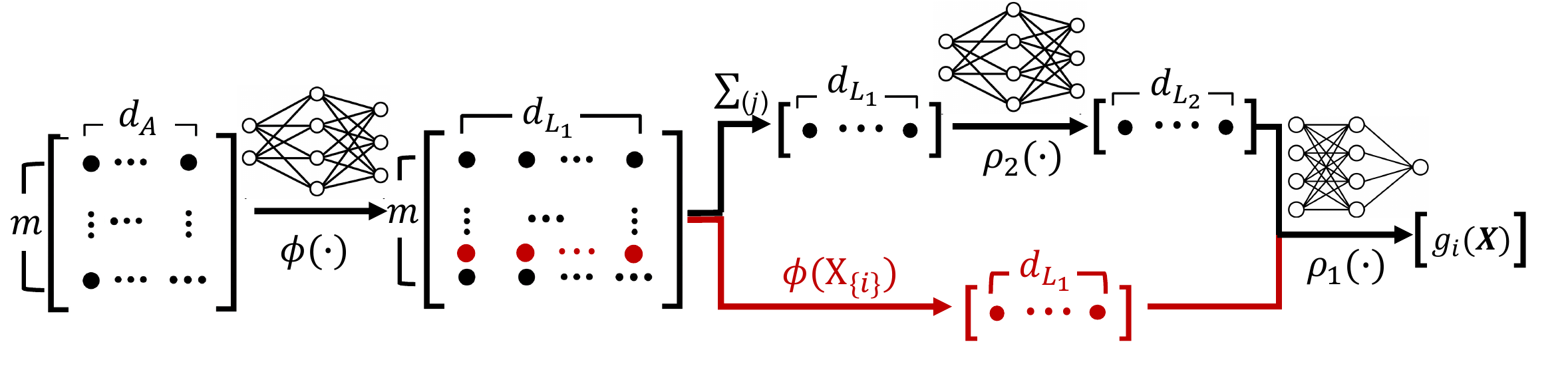}
    \end{adjustbox}
    \caption{Illustration of the architecture of $\DL^{\mu}$}
    \label{fig:W_mu}
\end{figure}
Here, $\sum_{(j)}$ denotes summation over all rows except the $j$th row, while $\phi(\bX_{\{i\}})$ extracts the latent representation corresponding to the $i$th input vector. The quantities $d_{L_{1}}$ and $d_{L_{2}}$ denote the dimensions of two latent spaces. Although these dimensions need not be equal in general, we set $d_{L_{1}} = d_{L_{2}}$ throughout this paper for simplicity and without architecture-specific fine-tuning. In Figure~\ref{fig:W_mu}, the inputs to $\rho_{1}$ are taken to be $\phi(\bX_{\{i\}})$ together with the output of $\rho_{2}$, rather than the original input $\bX_{\{i\}}$ as outlined in \eqref{equ:perm_preserve}. This design choice is motivated by the fact that $\phi(\bX_{\{i\}})$ lies in a higher-dimensional latent space, typically with $d_{L_{1}} > d_A$, thereby providing a richer and more expressive representation of the local input structure, which may enhance the flexibility and representational capacity of $\DL^{\mu}$.

The proposed neural Vecchia approximation also offers favorable computational scalability compared to classical Vecchia-type GP approximations. Both approaches scale linearly with respect to the number of training observations $N$. However, neural Vecchia additionally achieves linear complexity in the conditioning size $m$, whereas traditional Vecchia-type methods typically incur cubic complexity in $m$ due to repeated matrix factorizations and inversions. A similar cubic dependence on $m$ also appears in pseudo-input-based approaches \citep{SnelsonGhahramani2006}, where $m$ corresponds to the number of inducing or pseudo inputs. Consequently, neural Vecchia provides a computational advantage over existing scalable GP approximations, particularly when larger conditioning sets are required to capture complex dependence structures.

\section{Simulation Study}
\label{sec:sim_study}

\subsection{Benchmarks and scenarios}
\label{subsec:sim_scene}

In this section, we evaluate the proposed neural Vecchia kernel using simulated Vecchia-type Gaussian process (GP) realizations and compare it against three established covariance kernels across four covariance structures used for data simulation. For a given conditioning size $m$, we generate $N$ independent batches of $m+1$ pairs of input locations and responses
\begin{equation}
  \mbox{input location:}\; \bx_{i}^{(r)} \in \mathbb{R}^{d}, \;\; \mbox{response:}\; y_{i}^{(r)}, \;\; i = 1, \ldots, m+1,\; r = 1, \ldots, N   
\end{equation}
with $d=3$. For each $r$, $\{\bx_{i}^{(r)}\}_{i = 1}^{m+1}$ is drawn using Latin hypercube sampling and $\{y_{i}^{(r)}\}_{i = 1}^{m+1}$ is one realization of a zero-mean GP over $\{\bx_{i}^{(r)}\}_{i = 1}^{m+1}$ under a specified covariance kernel. The goal is to predict $y^{(r)}_{m+1}$ using $\{\bx_{i}^{(r)}\}_{i = 1}^{m+1}$ and $\{y_{i}^{(r)}\}_{i = 1}^{m}$ for each $r$, for which we use the negative log-likelihood (NLL) as our objective function. These simulated batches are used to train both the proposed kernel and the benchmark kernels. Performance is evaluated on an independent test set consisting of $N^*$ batches generated in the same manner.

As benchmarks, we consider three covariance kernels: the Matérn kernel with smoothness $1.5$ and automatic relevance determination (MT15), the spectral mixture kernel (SM) \citep{WilsonAdams2013}, and the deep kernel of \citet{Wilson2016DeepKernel} (DTSM), which combines neural network-based input transformation with the SM kernel. We refer to our proposed method as \textbf{NeuVec}. The SM and DTSM models follow the specifications in Table~1 of \citet{Wilson2016DeepKernel}: SM uses six spectral components, while DTSM employs a fully connected neural network with architecture [$d$-1000-1000-500-50-2] to transform the inputs. For NeuVec, all components in $\DL^{\mu}$ and $\DL^{\sigma}$ are parameterized by fully connected neural networks, with architectures detailed in Table~\ref{tbl:architecture}.
\begin{table}[h]
  \caption{Architecture of NeuVec for simulation study}
  \label{tbl:architecture}
  \centering
  \begin{tabular}{l| l | l | l}
    \toprule
    $\phi$ of $\DL^{\mu}$ & [$2d$, 128, 128, 128, 64] & $\phi$ of $\DL^{\sigma}$ & [$2d$, 128, 128, 128, 64] \\
    \midrule
    $\rho_{1}$ of $\DL^{\mu}$ & [128, 128, 128, 128, 1] & $\rho$ of $\DL^{\sigma}$ & [64, 128, 128, 128, 1]\\
    \midrule
    $\rho_{2}$ of $\DL^{\mu}$ & [64, 128, 128, 128, 64] \\
    \bottomrule
  \end{tabular}
\end{table}

For data generation, we consider a set of covariance kernels with varying levels of complexity. These include the Matérn kernel with smoothness 1.5 and automatic relevance determination (MT15), a lengthscale-nonstationary kernel (RangeNS) introduced in \citet{Cao2024}, a Periodic kernel, and a domain-transformed Matérn kernel (DTMT15); see Table~\ref{tbl:sim_kernels} for their parameterizations. 
\begin{table}[h]
\centering
\caption{Covariance kernel parameterizations used in data simulation}
\label{tbl:sim_kernels}
\small
\setlength{\tabcolsep}{4pt}
\begin{tabularx}{\textwidth}{
    >{\raggedright\arraybackslash}p{0.18\textwidth}
    >{\raggedright\arraybackslash}X
    >{\raggedright\arraybackslash}p{0.28\textwidth}
}
\toprule
Kernel & $\kernel(\bx_{i},\bx_{j})$ & Parameters \\
\midrule

MT15
& $\sigma^{2} M_{\nu}\!\left(\left\|\tfrac{\bx_{i} - \bx_{j}}{\boldsymbol{\ell}}\right\|\right) + \tau^{2} \delta_{ij}$ 
& $\sigma=1.0,\; \ell_{q}=0.3$ \\
& & $\nu=1.5,\; \tau^{2}=0.01$\\

\midrule

RangeNS
& $\sigma^{2} c_{ij} \exp(-h_{ij}) + \tau^{2} \delta_{ij}$\\
& $c_{ij} = \left( \frac{\ell(\bx_{i})^{1/4}\,\ell(\bx_{j})^{1/4}}{\left(\tfrac{\ell(\bx_{i})+\ell(\bx_{j})}{2}\right)^{1/2}} \right)^{d}$ &
$\beta_{0}=-2.0,\; \beta_{1}=1.0,$ \\
& $h_{ij} = \frac{\|\bx_{i} - \bx_{j}\|}{\left(\tfrac{\ell(\bx_{i})+\ell(\bx_{j})}{2}\right)^{1/2}}$ 
& $\beta_{2}=-1.0$ \\
& $\ell(\bx) = \exp\!\big(\beta_{0} + \beta_{1} \sin(3\pi s) + \beta_{2} \cos(2\pi s)\big)$ 
& $\sigma=1.0,\; \tau^{2}=0.01$  \\
& $s=\mathbf{1}^{\top} \bx$ 
\\

\midrule

Periodic 
& $\sigma^{2} \exp\!\left(-\frac{2}{\ell} \sum_{q} \sin^{2}\!\left(\frac{\pi (x_{i q}-x_{j q})}{p}\right)\right) + \tau^{2} \delta_{ij}$ 
& $\sigma=1.0,\; \ell=d$ \\
& & $p=0.5,\; \tau^{2}=0.01$ \\

\midrule

DTMT15
& $\sigma^{2} M_{\nu}\!\left(\|\bA\bx_{i} - \bA\bx_{j}\|\right) + \tau^{2} \delta_{ij}$ 
& $\sigma=1.0,\; \ell=\sqrt{d}/10$ \\
& $\bA$ is a $d \times d$ magic square
& $\nu=1.5,\; \tau^{2}=0.01$ \\

\bottomrule
\end{tabularx}
\end{table}
In Table~\ref{tbl:sim_kernels}, $\delta_{ij}$ denotes the Kronecker delta function, $M_{\nu}(\cdot)$ denotes the Mat\'ern correlation function with smoothness $\nu$, and $\boldsymbol{\ell} = (\ell_{1},\ldots,\ell_{d})$ represents lengthscales, with $\ell_{q}$ controlling the scale of the $q$th input dimension. The MT15 kernel serves as a stationary baseline, while RangeNS introduces input-dependent lengthscales to capture nonstationarity. The Periodic kernel models repeating patterns and DTMT15 adds geometric flexibility through domain transformation. Together, these kernels span a diverse range of covariance structures for our simulation study.

\subsection{Numerical results}

We train all models using mini-batch optimization with batch size $N^{\text{\scriptsize batch}} = 2{,}048$ for $N^{\text{\scriptsize iter}} = 60{,}000$ iterations, resulting in a total of $N = N^{\text{\scriptsize batch}} \times N^{\text{\scriptsize iter}}$ training samples. The test set size is set to $N^* = 10 N^{\text{\scriptsize batch}}$, which, although much smaller than $N$, is sufficient to provide stable estimates of the average loss for Vecchia-type GP realizations. Optimization is performed using the AdamW algorithm implemented in the PyTorch framework \citep{paszke2019pytorch} with an initial learning rate of $10^{-3}$, weight decay of $10^{-4}$, and a scheduler that reduces the learning rate to $10^{-5}$ by the end of training. As training data can be generated on demand, no regularization such as dropout is applied.

We first compare the proposed NeuVec with the three benchmark kernels at $m=30$. For reference, we also report the performance obtained when the true covariance kernel (i.e., the one used for data simulation) is used within a Vecchia GP. Model performance is evaluated using mean squared error (MSE) and negative log-likelihood (NLL), with results summarized in Table~\ref{tbl:sim_m30}.
\begin{table}[h]
\small
\caption{Performance of four covariance kernels (columns) under four simulation scenarios (rows) at $m = 30$. The column `True' reports the statistics when the true kernel was used within a Vecchia GP.}
\label{tbl:sim_m30}
\centering
\vspace{-10pt}
\begin{adjustbox}{center}
\begin{tabular}[t]{lllllrllllr}
\toprule
\multicolumn{1}{c}{ } & \multicolumn{5}{c}{MSE} & \multicolumn{5}{c}{NLL} \\
\cmidrule(l{3pt}r{3pt}){2-6} \cmidrule(l{3pt}r{3pt}){7-11}
Kernel & MT15 & SM & DTSM & NeuVec & True & MT15 & SM & DTSM & NeuVec & True\\
\midrule
MT15 & \textbf{0.391} & 0.401 & 0.503 & 0.408 & 0.401 & \textbf{0.877} & 0.894 & 1.046 & 0.971 & 0.885\\
RangeNS & 0.675 & 0.687 & 0.690 & \textbf{0.609} & 0.608 & 1.213 & 1.225 & 1.215 & \textbf{1.171} & 1.132\\
Periodic & 0.647 & 0.674 & 0.583 & \textbf{0.284} & 0.134 & 1.164 & 1.180 & 1.136 & \textbf{0.789} & 0.322\\
DTMT15 & 0.427 & 0.440 & 0.410 & \textbf{0.335} & 0.331 & 0.934 & 0.953 & 0.921 & \textbf{0.872} & 0.768\\
\bottomrule
\end{tabular}
\end{adjustbox}
\end{table}
The proposed NeuVec kernel achieves the lowest MSE and NLL across all scenarios except MT15. As expected, the MT15 kernel performs best when the data are generated from the same MT15 structure, with the SM kernel slightly outperforming NeuVec in this setting. This is unsurprising as MT15 corresponds to the simplest scenario, stationary and isotropic, among the four. In contrast, under data generated by more complex covariance models, NeuVec substantially outperforms all benchmark kernels. Moreover, with the exception of the Periodic case, its performance is close to that obtained using the true covariance kernel, which is typically unavailable in practice.

We further examine how performance varies with the conditioning size $m$, with the expectation that larger $m$ leads to improved accuracy. Specifically, we vary $m$ from $10$ to $90$ in increments of $20$, keeping all other settings the same as Table~\ref{tbl:sim_m30}. All experiments are conducted on a Tesla P100-PCIE-16GB GPU with a maximum training time of 8 hours per configuration (defined by the kernel for simulation, the kernel for training, and $m$). Under this constraint, the DTSM kernel does not complete the full $60{,}000$ training iterations when $m \ge 50$. The resulting MSE trends are shown in Figure~\ref{fig:sim_MSE}, while the corresponding NLL results, which share similar patterns as MSE, are reported in Figure~\ref{fig:sim_NLL} in the Appendix.
\begin{figure}[t]
  \centering
  \begin{adjustbox}{center}
  \includegraphics[width=1.0\textwidth]{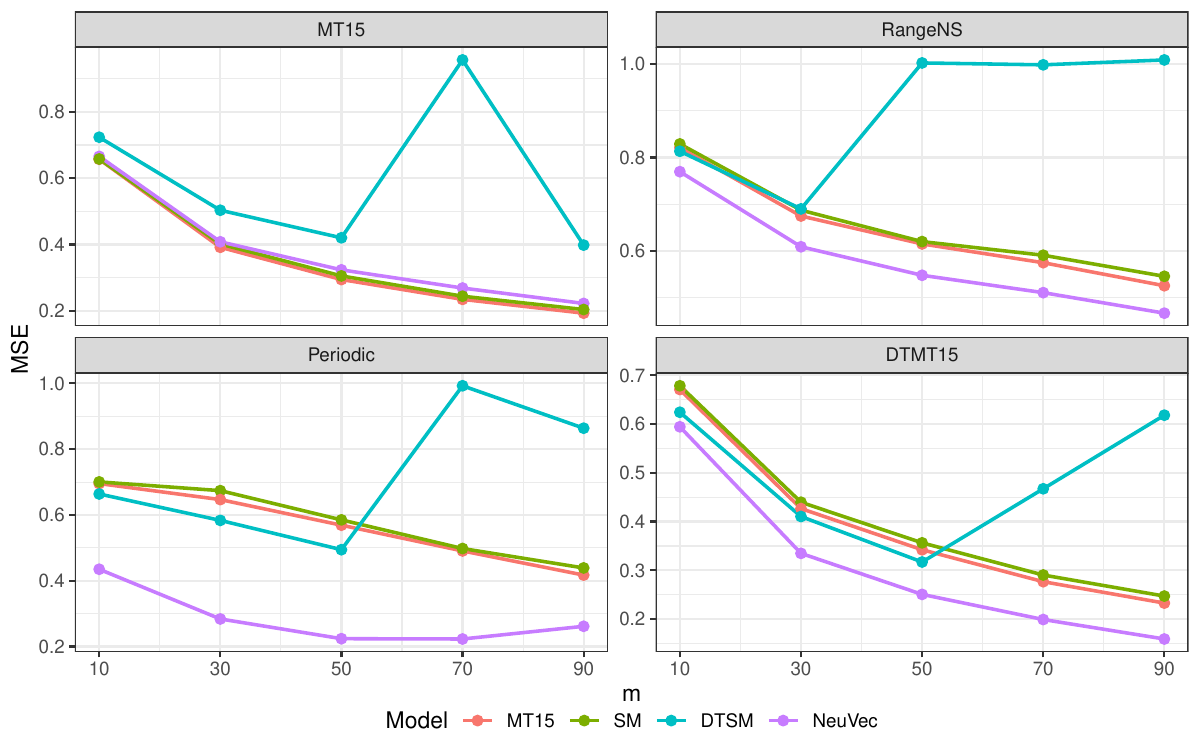}
  \end{adjustbox}
  \caption{Predictive MSE of using the proposed NeuVec kernel and the benchmarks. Each panel corresponds to one simulation scenario.}
  \label{fig:sim_MSE}
\end{figure}
NeuVec maintains a clear advantage over the benchmark methods under the RangeNS, Periodic, and DTMT15 scenarios, while remaining competitive with MT15 and SM in the simpler MT15 setting. Overall, its prediction error decreases consistently as $m$ increases, indicating that NeuVec effectively leverages larger conditioning sets and remains robust to covariance misspecification. DTSM exhibits degraded performance at $m = 70$ and $m = 90$, likely due to early termination of training. In terms of computational cost, the time per 1,000 iterations at $m = 90$ is approximately 30, 215, 10,000, and 69 seconds for MT15, SM, DTSM, and NeuVec, respectively.

\section{Application Study}
\label{sec:application}
We evaluate the proposed method using temperature data from the Argo program \citep{argo2000}, a global array of profiling floats that measure oceanographic variables such as temperature and salinity throughout the upper ocean. We focus on observations collected between February 2010 and 2026 within the depth range of 0-500 meters and a spatial region defined by longitude $[-75^{\circ}, -45^{\circ}]$ and latitude $[30^{\circ}, 45^{\circ}]$. This region corresponds to the Gulf Stream, a dynamically complex and highly variable ocean current, making it a challenging testbed for covariance modeling due to its strong nonstationarity and heterogeneous spatial structure.

We treat temperature as the response variable and model the temperature field across different years as independent realizations from a common Gaussian process. To balance the data across floats, we perform subsampling within each year: if a float records more than 200 temperature measurements, we randomly select 200 observations from that float. The input variables consist of longitude, latitude, depth, and day of the year. After subsampling, all input variables are normalized to the range $[0,1]$ based on their values across the 17-year period (Feb only), while the response variable (temperature) is left unnormalized. This preprocessing yields a total of $184{,}706$ input-response pairs.

For each year, the data are randomly split into 80\% for training and 20\% for testing. Conditioning sets are constructed using the $m=30$ nearest neighbors in the scaled input space, restricted to observations from the same year and excluding measurements from the same float, following standard practice in the Argo literature (e.g., \citet{KuuselaStein18, BaughMcKinnon2022}). To account for differing correlation scales across the four input dimensions, we first fit a Vecchia GP with an MT15 kernel to obtain initial estimates of the lengthscales. These estimates are then used to rescale the input space for nearest-neighbor selection. Compared to the simulation study in Section~\ref{sec:sim_study}, the available training data are more limited, and we therefore adopt a reduced NeuVec architecture, described in Table~\ref{tbl:architecture_argo}, to mitigate overfitting.
\begin{table}[h]
  \caption{Architecture of NeuVec for Argo data}
  \label{tbl:architecture_argo}
  \centering
  \begin{tabular}{l| l | l | l}
    \toprule
    $\phi$ of $\DL^{\mu}$ & [$2d$, 16, 16, 16, 16] & $\phi$ of $\DL^{\sigma}$ & [$2d$, 16, 16, 16, 16] \\
    \midrule
    $\rho_{1}$ of $\DL^{\mu}$ & [32, 16, 16, 16, 1] & $\rho$ of $\DL^{\sigma}$ & [16, 16, 16, 16, 1]\\
    \midrule
    $\rho_{2}$ of $\DL^{\mu}$ & [16, 16, 16, 16, 16] \\
    \bottomrule
  \end{tabular}
\end{table}
We adopt the same configurations for MT15, SM, and DTSM as in the simulation study. For MT15 and SM, the number of parameters is relatively small compared to the training sample size, almost eliminating the risk of overfitting. Although DTSM involves a large number of parameters, its deep-learning component outputs only a two-dimensional representation that is subsequently passed to the SM kernel, providing an implicit form of regularization. Moreover, the original work of \citet{zaheer2017deep} employs a similar DTSM architecture for datasets of size over $6{,}000$, supporting its use in our setting.

For the Argo dataset, we incorporate a shared mean function across all methods, modeled by a simple fully connected neural network with architecture [$d$, 8, 8, 8, 1], to facilitate extrapolation across floats. To further mitigate overfitting, we apply dropout with rate 0.3 to both NeuVec and the mean functions of all methods. Training is conducted using the same batch size ($2{,}048$) and number of iterations ($60{,}000$) as in Section~\ref{sec:sim_study}, and all models complete optimization within the 8-hour time budget. Table~\ref{tbl:argo} summarizes the predictive performance of the four methods in terms of MSE and NLL.
\begin{table}[h]
\caption{Prediction accuracy for Argo dataset.}
\label{tbl:argo}
\centering
\vspace{-10pt}
\begin{tabular}[t]{llllllll}
\toprule
\multicolumn{4}{c}{MSE} & \multicolumn{4}{c}{NLL} \\
\cmidrule(l{3pt}r{3pt}){1-4} \cmidrule(l{3pt}r{3pt}){5-8}
MT15 & SM & DTSM & NeuVec & MT15 & SM & DTSM & NeuVec\\
\midrule
4.813 & 4.915 & 5.132 & \textbf{3.546} & 2.195 & 2.208 & \textbf{1.315} & 2.052\\
\bottomrule
\end{tabular}
\end{table}
NeuVec achieves substantially lower MSE than all three benchmark methods, highlighting the advantage of its flexible, nonparametric covariance representation for predictive accuracy. In contrast, DTSM attains the best NLL, indicating stronger performance in uncertainty quantification. This advantage likely stems from DTSM’s balance between expressiveness and implicit regularization, whereas the fully nonparametric nature of NeuVec may require larger datasets to fully calibrate predictive uncertainty.

\section{Conclusion and Discussion}
\label{sec:discuss}

In this paper, we propose a neural Vecchia Gaussian process framework, termed NeuVec, which parameterizes covariance structures under the Vecchia approximation using deep neural networks. The proposed representation offers several important advantages. First, it fully leverages the expressive power of deep learning architectures by directly learning the covariance structure itself, rather than learning parameters of a predefined covariance family or coefficients of basis expansions. This is possible because the inverse Cholesky factor induced by the Vecchia representation is unconstrained, in contrast to covariance or precision matrices, which must satisfy positive-definiteness constraints. Second, the kriging-coefficient and conditional-standard-deviation parameterization leads to a streamlined optimization procedure involving only regression-type quantities, thereby avoiding repeated matrix inversions or factorizations during training. Third, the learning targets are intrinsic properties of the covariance structure and do not depend on randomness in the GP realizations. Consequently, the proposed approach provides more stable and efficient learning than directly modeling conditional means.

We further develop a universal representation theorem for permutation-preserving functions, which includes the mapping from conditioning locations to kriging coefficients as a special case. This result guides the proposed neural architecture and may also be useful more broadly in the design of equivariant learning systems. Through simulation studies, we demonstrate that NeuVec substantially outperforms existing covariance parameterizations under complex nonstationary covariance structures while remaining competitive under simpler stationary settings. We also evaluate NeuVec on the Argo temperature dataset over the Gulf Stream region, where it achieves superior predictive accuracy relative to benchmark methods.

The proposed NeuVec framework also has its tradeoffs. First, as with all Vecchia-type approximations, the induced covariance structure depends on the ordering of observations and the choice of conditioning sets. Consequently, changing the ordering or conditioning structure while keeping the learned NeuVec representation fixed may lead to slightly different covariance approximations. In practice, however, this sensitivity is expected to diminish as the training dataset becomes larger and the learned representation becomes more stable. Second, because NeuVec is highly flexible and largely nonparametric, it imposes weaker structural regularization than conventional parametric or semi-parametric kernels. As a result, larger datasets may be required to fully realize its representational capacity while avoiding overfitting. For smaller datasets, overfitting can be mitigated through architectural simplification, dropout, or other regularization techniques.

An interesting direction for future work is to combine NeuVec with structured covariance priors or parsimonious kernels to be more robust against overfitting. For example, one may augment training using realizations simulated from a simpler covariance model, such as a Mat\'ern kernel, thereby introducing an inductive bias toward stable covariance structures. Additional directions include developing principled regularization strategies for NeuVec and extending training to varying conditioning sizes $m$. The source code for replicating the results in this paper is available on Github.

\bibliographystyle{apalike}
\bibliography{refs}

@article{Vecchia1988,
    title = {{Estimation and model identification for continuous spatial processes}},
    year = {1988},
    journal = {Journal of the Royal Statistical Society, Series B},
    author = {Vecchia, AV},
    number = {2},
    pages = {297--312},
    volume = {50},
    url = {http://www.jstor.org/stable/10.2307/2345768}
}

@article{Datta2016,
    title = {{Hierarchical nearest-neighbor {G}aussian process models for large geostatistical datasets}},
    year = {2016},
    journal = {Journal of the American Statistical Association},
    author = {Datta, Abhirup and Banerjee, Sudipto and Finley, Andrew O. and Gelfand, Alan E.},
    number = {514},
    pages = {800--812},
    volume = {111},
    url = {http://arxiv.org/abs/1406.7343},
    doi = {10.1080/01621459.2015.1044091},
    issn = {0162-1459},
    arxivId = {arXiv:1406.7343v1}
}

@article{Katzfuss2017a,
    title = {{A general framework for {V}ecchia approximations of {G}aussian processes}},
    year = {2021},
    journal = {Statistical Science},
    author = {Katzfuss, Matthias and Guinness, Joseph},
    number = {1},
    pages = {124--141},
    volume = {36},
    url = {http://arxiv.org/abs/1708.06302},
    doi = {10.1214/19-STS755},
    arxivId = {1708.06302}
}

@article{cao2025linear,
  title={Linear-cost Vecchia approximation of multivariate normal probabilities},
  author={Cao, Jian and Katzfuss, Matthias},
  journal={Journal of the American Statistical Association},
  pages={1--14},
  year={2025},
  publisher={Taylor \& Francis}
}

@article{zaheer2017deep,
  title={Deep sets},
  author={Zaheer, Manzil and Kottur, Satwik and Ravanbakhsh, Siamak and Poczos, Barnabas and Salakhutdinov, Russ R and Smola, Alexander J},
  journal={Advances in neural information processing systems (NeurIPS)},
  volume={30},
  year={2017}
}

@book{RasmussenWilliams,
  author    = {Carl Edward Rasmussen and Christopher K. I. Williams},
  title     = {Gaussian Processes for Machine Learning},
  publisher = {MIT Press},
  year      = {2006}
}

@article{WilsonAdams2013,
  author  = {Andrew Gordon Wilson and Ryan Prescott Adams},
  title   = {Gaussian Process Kernels for Pattern Discovery and Extrapolation},
  journal = {Proceedings of the 30th International Conference on Machine Learning (ICML)},
  year    = {2013}
}

@article{SampsonGuttorp1992,
  author  = {Paul D. Sampson and Peter Guttorp},
  title   = {Nonparametric Estimation of Nonstationary Spatial Covariance Structure},
  journal = {Journal of the American Statistical Association},
  volume  = {87},
  number  = {417},
  pages   = {108--119},
  year    = {1992}
}

@article{PaciorekSchervish2006,
  author  = {Christopher J. Paciorek and Mark J. Schervish},
  title   = {Spatial Modelling Using a New Class of Nonstationary Covariance Functions},
  journal = {Environmetrics},
  volume  = {17},
  number  = {5},
  pages   = {483--506},
  year    = {2006}
}

@article{SnelsonGhahramani2006,
  author  = {Edward Snelson and Zoubin Ghahramani},
  title   = {Sparse Gaussian Processes Using Pseudo-inputs},
  journal = {Advances in Neural Information Processing Systems (NeurIPS)},
  volume={18},
  year    = {2006}
}

@article{Titsias2009,
  author  = {Michalis K. Titsias},
  title   = {Variational Learning of Inducing Variables in Sparse Gaussian Processes},
  journal = {Proceedings of the 12th International Conference on Artificial Intelligence and Statistics (AISTATS)},
  year    = {2009}
}

@article{WilsonNickisch2015,
  author  = {Andrew Gordon Wilson and Hannes Nickisch},
  title   = {Kernel Interpolation for Scalable Structured Gaussian Processes ({KISS-GP})},
  journal = {Proceedings of the 32nd International Conference on Machine Learning (ICML)},
  year    = {2015}
}

@inproceedings{Wilson2016DeepKernel,
  author    = {Andrew Gordon Wilson and Zhiting Hu and Ruslan Salakhutdinov and Eric P. Xing},
  title     = {Deep Kernel Learning},
  booktitle = {Proceedings of the 19th International Conference on Artificial Intelligence and Statistics (AISTATS)},
  year      = {2016}
}

@article{Wilson2016Stochastic,
  title={Stochastic variational deep kernel learning},
  author={Wilson, Andrew G and Hu, Zhiting and Salakhutdinov, Russ R and Xing, Eric P},
  journal={Advances in neural information processing systems (NeurIPS)},
  volume={29},
  year={2016}
}

@article{Chen2021DeepKriging,
  title={DeepKriging: Spatially Dependent Deep Neural Networks for Spatial Prediction},
  author={Chen, Wanfang and Li, Yuxiao and Reich, Brian J and Sun, Ying},
  journal={Statistica Sinica},
  volume={34},
  number={1},
  pages={291--311},
  year={2024},
  publisher={JSTOR}
}

@article{Lindgren2011a,
    title = {{An explicit link between Gaussian fields and Gaussian Markov random fields: the stochastic partial differential equation approach}},
    year = {2011},
    journal = {Journal of the Royal Statistical Society, Series B},
    author = {Lindgren, Finn and Rue, Håvard and Lindstr{\"{o}}m, J},
    number = {4},
    pages = {423--498},
    volume = {73}
}

@article{sainsbury-dale_likelihood-free_2024,
	title = {Likelihood-{Free} {Parameter} {Estimation} with {Neural} {Bayes} {Estimators}},
	volume = {78},
	issn = {0003-1305},
	url = {https://doi.org/10.1080/00031305.2023.2249522},
	doi = {10.1080/00031305.2023.2249522},
	number = {1},
	urldate = {2026-01-07},
	journal = {The American Statistician},
	author = {Sainsbury-Dale, Matthew and Zammit-Mangion, Andrew and Huser, Raphaël},
	month = jan,
	year = {2024},
	pages = {1--14}
}

@article{lenzi_neural_2023,
	title = {Neural networks for parameter estimation in intractable models},
	volume = {185},
	issn = {0167-9473},
	url = {https://www.sciencedirect.com/science/article/pii/S0167947323000737},
	doi = {10.1016/j.csda.2023.107762},
	urldate = {2026-01-06},
	journal = {Computational Statistics \& Data Analysis},
	author = {Lenzi, Amanda and Bessac, Julie and Rudi, Johann and Stein, Michael L.},
	month = sep,
	year = {2023},
	pages = {107762}
}

@Article{Cao2024,
author={Cao, Jian
and Zhang, Jingjie
and Sun, Zhuoer
and Katzfuss, Matthias},
title={Locally Anisotropic Nonstationary Covariance Functions on the Sphere},
journal={Journal of Agricultural, Biological and Environmental Statistics},
year={2024},
month={Jun},
day={01},
volume={29},
number={2},
pages={212-231},
issn={1537-2693},
doi={10.1007/s13253-023-00573-y},
url={https://doi.org/10.1007/s13253-023-00573-y}
}

@article{KuuselaStein18,
author = {Mikael Kuusela and Michael L. Stein},
    title = {Locally Stationary Spatio-Temporal Interpolation of Argo Profiling Float Data.},
    journal = {Proceedings of the Royal Society A: Mathematical, Physical and Engineering Sciences},
doi = {10.1098/rspa.2018.0400},
volume={474},
issue={2220},
    year = {2018}
}

@misc{argo2000,
  author = {Argo},
  title = {Argo float data and metadata from {Global Data Assembly Centre} ({Argo GDAC})},
  year = {2000},
  publisher = {SEANOE},
  doi = {10.17882/42182},
  note = {\url{https://doi.org/10.17882/42182}}
}

@article{BaughMcKinnon2022,
author = {Samuel Baugh and Karen McKinnon},
title = {{Hierarchical Bayesian modeling of ocean heat content and its uncertainty}},
volume = {16},
journal = {The Annals of Applied Statistics},
number = {4},
publisher = {Institute of Mathematical Statistics},
pages = {2603 -- 2625},
year = {2022},
doi = {10.1214/22-AOAS1605}
}

@article{zhu2025scalable,
  title={Scalable Gaussian Processes with Low-Rank Deep Kernel Decomposition},
  author={Zhu, Yunqin and Yuchi, Henry Shaowu and Xie, Yao},
  journal={arXiv preprint arXiv:2505.18526},
  year={2025}
}

@article{snelson2003warped,
  title={Warped gaussian processes},
  author={Snelson, Edward and Ghahramani, Zoubin and Rasmussen, Carl},
  journal={Advances in neural information processing systems (NeurIPS)},
  volume={16},
  year={2003}
}

@article{schafer2021sparse,
  title={Sparse Cholesky Factorization by {K}ullback--{L}eibler Minimization},
  author={Schäfer, Florian and Katzfuss, Matthias and Owhadi, Houman},
  journal={SIAM Journal on Scientific Computing},
  volume={43},
  number={3},
  pages={A2019--A2046},
  year={2021},
  publisher={SIAM}
}

@article{paszke2019pytorch,
  title={Pytorch: An imperative style, high-performance deep learning library},
  author={Paszke, Adam and Gross, Sam and Massa, Francisco and Lerer, Adam and Bradbury, James and Chanan, Gregory and Killeen, Trevor and Lin, Zeming and Gimelshein, Natalia and Antiga, Luca and others},
  journal={Advances in neural information processing systems},
  volume={32},
  year={2019}
}


\appendix

\section{Proofs}
\label{app:proofs}

\begin{proof}[Proof of Proposition~\ref{prp:perm_invariant}]
    The Theorem~7 of \citet{zaheer2017deep} corresponds to the special case of $k_{1} = k_{2} = 1$. We first consider the case where $k_{1} = 1$ and $k_{2} \ge 1$. Same as the Theorem~7 of \citet{zaheer2017deep}, define 
    \[
    \phi(\bX_{\{i\}}) = [1, \bX_{\{i\}}, \bX_{\{i\}}^{2}, \ldots, \bX_{\{i\}}^{m}].
    \]
    Define $E(\bX) = \sum_{i = 1}^{m} \phi(\bX_{\{i\}})$ and $\mathcal{Z}$ as the image of the domain of $\bg$ under the function $E$. Define $\mathcal{X} = \{(\bX_{\{1\}}, \bX_{\{2\}}, \ldots, \bX_{\{m\}}) \mid \bX_{\{1\}} \le \bX_{\{2\}} \le \cdots \le \bX_{\{m\}}\}$. $E: \mathcal{X} \rightarrow \mathcal{Z}$ is injective with a continuous inverse. Define $\rho(\bz) = \bg(E^{-1}(\bz))$, we can express $\bg(\bX)$ as $\rho\big(\sum_{i = 1}^{m} \phi(\bX_{\{i\}})\big)$.

    When $k_{1} \ge 1$ and $k_{2} \ge 1$, the proof of the Theorem~9 in \citet{zaheer2017deep} suggests that the entries of $\bg$ can all be approximated arbitrarily close by polynomials over the compact domain of $\bg$. For each $g_i$, due to permutation-invariance, its approximation polynomial can be written as a polynomial of homogeneous symmetric monomials and homogeneous symmetric monomials are the sum over monomial terms. Therefore, we can define $\phi(\bX_{\{i\}})$ as a collection of monomials up to a certain order and the entries of $\rho$ are polynomials corresponding to the approximation polynomials of the entries of $\bg$.
 \end{proof}

\begin{proof}[Proof of Theorem~\ref{thm:perm_preserving}]
    We first show that the vector-valued function $\tilde{\bg} = [\tilde{g}_1, \tilde{g}_2, \ldots, \tilde{g}_m]^{\top}$, where
    \[
    \tilde{g}_{i}(\bX) = \rho_{1}\left(\bX_{\{i\}}, \rho_{2}\left(\sum_{j \in \{1, \ldots, m\} \setminus \{i\}} \phi(\bX_{\{j\}})\right)\right),
    \]
    satisfies permutation-preserving property. Denote the permutation function defined by a permutation matrix $\bP \in \mathbb{R}^{m \times m}$ by $\pi: \{1, \ldots, m\} \rightarrow \{1, \ldots, m\}$ with $(i, \pi(i))$ for $i = 1, \ldots, m$ being the indices for the non-zero entries of $\bP$. Recall that $\bX_{\{i\}}$ denotes the $i$th row of $\bX$. Consider the $i$th entry of $\bP \tilde{\bg}(\bX)$:
    \[
    \tilde{g}_{\pi(i)}(\bX) = \rho_{1}\left(\bX_{\{\pi(i)\}}, \rho_{2}\left(\sum_{j \in \{1, \ldots, m\} \setminus \{\pi(i)\}} \phi(\bX_{\{j\}})\right)\right),
    \]
    and the $i$th entry of $\tilde{\bg}(\bP \bX)$
    \[
    \tilde{g}_{i}(\bP \bX) = \rho_{1}\left(\bX_{\{\pi(i)\}}, \rho_{2}\left(\sum_{j \in \{1, \ldots, m\} \setminus \{\pi(i)\}} \phi(\bX_{\{j\}})\right)\right).
    \]
    Hence, we show that $\tilde{\bg} = [\tilde{g}_1, \tilde{g}_2, \ldots, \tilde{g}_m]^{\top}$ is permutation-preserving.
    
    Next, we show that $g_{i}(\bX)$ can be approximated arbitrarily close by suitably chosen $\tilde{g}_{i}(\bX)$. We first show that $g_{i}$ and $g_{j}$, considered as functions in $\mathbb{R}^{md}$, i.e., $g_{i}, g_{j}: (\bX_{\{1\}}, \bX_{\{2\}}, \ldots, \bX_{\{m\}}) \rightarrow \mathbb{R}$, are symmetric against the hyperplane $\bX_{\{i\}} = \bX_{\{j\}}$, where the equality of two vectors is defined entry-wise. Without loss of generality, it suffices to show that $g_{1}$ and $g_{2}$ are symmetric against the hyperplane $\bX_{\{1\}} = \bX_{\{2\}}$. Set the permutation matrix $\bP$ as
    \[
    \bP = 
    \begin{bmatrix}
    0 & 1 & \bfzero \\
    1 & 0 & \bfzero \\
    \bfzero & \bfzero & \bI_{m-2} 
    \end{bmatrix},
    \]
    since $\bg$ is permutation-preserving, we have the relationship
    \[
    g_{1}(\bX_{\{2\}}, \bX_{\{1\}}, \bX_{\{3\}}, \ldots, \bX_{\{m\}}) = 
    g_{2}(\bX_{\{1\}}, \bX_{\{2\}}, \bX_{\{3\}}, \ldots, \bX_{\{m\}}).
    \]
    With a modification of the permutation matrix, we can derive that $g_{i}$ and $g_{j}$ are symmetric against the hyperplane $\bX_{\{i\}} = \bX_{\{j\}}$. 
    
    Define $\rho_{1}(\bX_{\{1\}}, \ldots, \bX_{\{m\}}) = g_{1}(\bX_{\{1\}}, \ldots, \bX_{\{m\}})$, therefore, 
    \begin{equation}
    \label{equ:gi_rho1}
    g_{i}(\bX_{\{1\}}, \ldots, \bX_{\{m\}}) = \rho_{1}(\bX_{\{i\}}, \bx_{2}, \ldots, \bx_{i - 1}, \bX_{\{1\}}, \bx_{i+1}, \ldots, \bX_{\{m\}}).    
    \end{equation}
    Denote $\rho_{1}(\bX_{\{1\}}, \ldots, \bX_{\{m\}})$ by $\rho_{1}(\bX_{\{1\}}, h(\bX_{\{2:m\}}))$, where $h(\bX_{\{2:m\}})$ can be a vector-valued function and $\bX_{\{2:m\}}$ refers to the 2nd to the $m$th (last) row of $\bX$. Since $\bg$ is assumed continuous, we also restrict $\rho_{1}$ and $h$ to be continuous functions. We next show that $h$ is permutation-invariant. This can be shown by choosing the permutation matrix $\bP$ as 
    \[
    \bP = 
    \begin{bmatrix}
        1 & \bfzero \\
        \bfzero & \bP^{'}
    \end{bmatrix},
    \]
    where $\bP^{'} \in \mathbb{R}^{(m - 1) \times (m - 1)}$ is an arbitrary permutation matrix. Since $\bg(\bP \bX) = \bP \bg(\bX)$, we have $g_{1}(\bP \bX) = g_{1}(\bX)$ and hence 
    \[\rho_{1}(\bX_{\{1\}}, h(\bP^{'} \bX_{\{2:m\}})) = \rho_{1}(\bX_{\{1\}}, h(\bX_{\{2:m\}})).\] 
    Because $\rho_{1}$ can be arbitrary kriging coefficient function and there is no requirement on the kriging coefficient for the resulting covariance matrix to be positive definite. $\rho_{1}$ can effectively be any continuous function. Furthermore, there is no requirement on $\bX_{\{1\}}$ and $\bX_{\{2:m\}}$ as long as they belong to their compact domain. We conclude that $h$ must be permutation-invariant. Using Proposition~\ref{prp:perm_invariant}, we know that $h(\bX_{\{2:m\}})$ can be approximated arbitrarily close by $\rho_{2}(\sum_{i = 2}^{m} \phi(\bX_{\{i\}}))$ for suitable transformations $\rho_{2}$ and $\phi$. 
    Since $\rho_{1}$ is a continuous function over a compact domain. $\rho_{1}(\bX_{\{1\}}, h(\bX_{\{2:m\}}))$ can also be approximated arbitrarily close by $\rho_{1}(\bX_{\{1\}}, \rho_{2}(\sum_{i = 2}^{m} \phi(\bX_{\{i\}})))$. Hence, we conclude that $g_{i}(\bX_{\{1\}}, \ldots, \bX_{\{m\}})$ can be approximated arbitrarily close by
    \[
    \rho_{1}\left(\bX_{\{i\}}, \rho_{2}\left(\sum_{j \in \{1, \ldots, m\} \setminus \{i\}} \phi(\bX_{\{j\}})\right)\right),
    \]
    for suitable transformations $\rho_{1}$, $\rho_{2}$, and $\phi$. It is also worth noticing that when $k_{1} = 1$, $g_{i}$ can be exactly represented by the expression above, which can be proved in the same manner.
\end{proof}

\section{Supplementary Plots for Section~\ref{sec:sim_study}}
\begin{figure}[h]
  \centering
  \begin{adjustbox}{center}
  \includegraphics[width=1.0\textwidth]{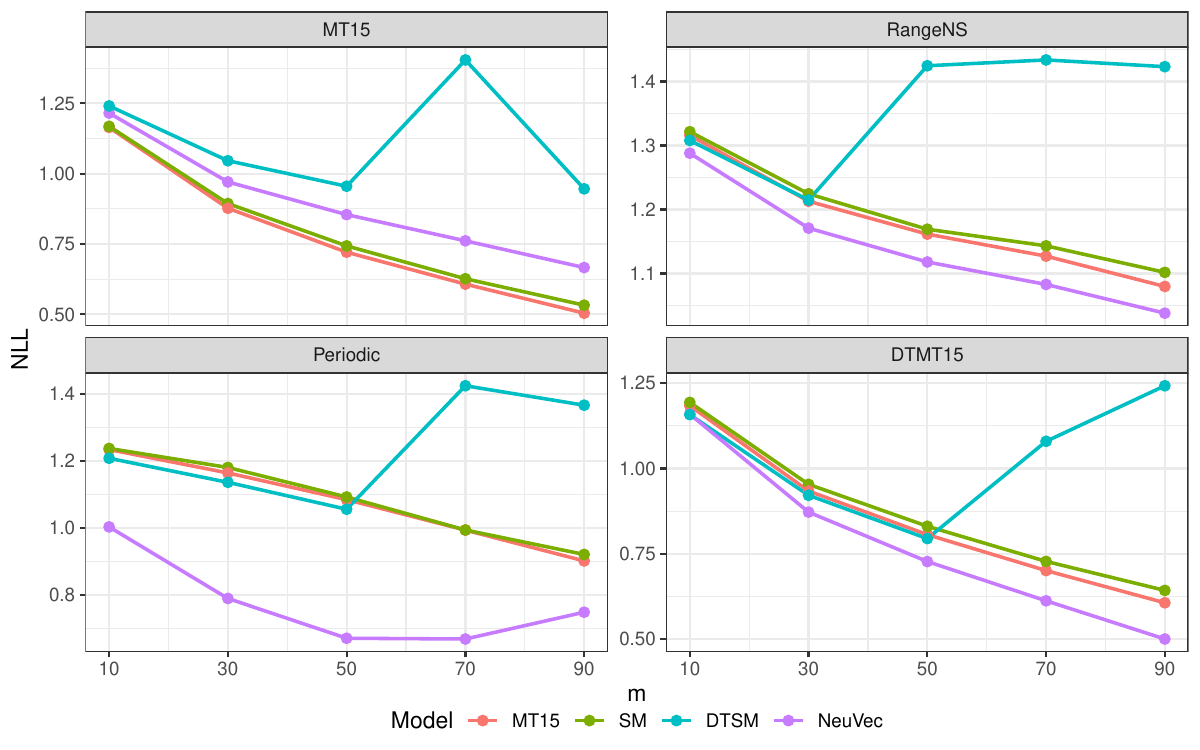}
  \end{adjustbox}
  \caption{Predictive NLL of using the proposed NeuVec kernel and the benchmarks. Each panel corresponds to one simulation scenario.}
  \label{fig:sim_NLL}
\end{figure}


\end{document}